\documentclass[letterpaper, 10 pt, conference]{ieeeconf}  

\IEEEoverridecommandlockouts                              

\usepackage{graphics} 
\usepackage{subcaption}
\usepackage{epsfig} 

\usepackage{amsthm}
\usepackage{amsmath} 
\usepackage{amssymb}  
\usepackage{algorithm}
\usepackage{algorithmic}
\usepackage{booktabs}
\usepackage{multirow}
\usepackage{graphicx}
\usepackage{cite}
\usepackage{hyperref}
\usepackage{censor}

\theoremstyle{definition}
\newtheorem{theorem}{Theorem}

\theoremstyle{definition}
\newtheorem{proposition}{Proposition}

\newtheorem*{remark}{Remark}

\title{\LARGE \bf
BridgeFlow: Fast and Robust SE(2)-Equivariant Motion \\ Planning with Flow Matching
}

\author{Xinzhe Zhou$^{\dagger}$, Xuyang Wang$^{\dagger}$, Xiaoming Duan$^{\dagger}$, and Jianping He$^{\dagger}$%
\thanks{This work was supported in part by the National Natural Science Foundation of China under Grants 625B1023, 62373247, and 62303314, and in part by the State Key Laboratory of Autonomous Intelligent Unmanned Systems under Grant ZZKF2025-1-19.}
\thanks{$^{\dagger}$All authors are with the School of Automation and Intelligent Sensing, Shanghai Jiao Tong University, Shanghai, China. E-mail: \{zhou\_xinzhe, 2674789652, xduan, jphe\}@sjtu.edu.cn.}}%

\begin{document}

\maketitle
\thispagestyle{empty}
\pagestyle{empty}

\begin{abstract}
In robotic motion planning, equivariance to rigid body transformations is crucial for robust spatial generalization. However, current learning-based planners face a critical dilemma: they either lack inherent equivariance, treating transformed tasks as novel scenarios, or enforce it via computationally expensive specialized architectures that bottleneck real-time inference. To break this trade-off, we propose BridgeFlow, a fast and strictly $\mathrm{SE(2)}$-equivariant generative motion planning framework. Rather than relying on heavy equivariant networks, BridgeFlow achieves exact spatial equivariance via a lightweight task-centric canonicalization module, enabling generalization using standard architectures. To further accelerate inference, we pair a Brownian bridge informative prior with context-aware mini-batch optimal transport. This constructs a straightened vector field that minimizes transport costs and stabilizes training. Furthermore, environmental awareness is explicitly embedded via Classifier-Free Guidance. Evaluations in dense 2D environments and on a 7-DoF Franka manipulator demonstrate that BridgeFlow achieves up to a $15\times$ inference speedup and a $2\times$ higher valid trajectory rate over state-of-the-art diffusion baselines, alongside robust generalization to entirely unseen environments and arbitrary spatial transformations.
\end{abstract}


\section{Introduction}

Motion planning is a fundamental problem in robotics, requiring the generation of collision-free, kinematically feasible trajectories from a start to a goal configuration in complex environments. For motion planning, equivariance to rigid body transformations is a fundamental cornerstone of generalization and robustness. It ensures that algorithms can reliably solve structurally identical tasks regardless of their global pose. While traditional sampling-based \cite{rrt, rrt-connect} and optimization-based \cite{chomp, stomp, trajopt, gpmp} algorithms provide theoretical guarantees and completeness, their computational bottlenecks severely limit real-time applicability. Consequently, learning-based approaches have emerged as a highly promising alternative, though effectively integrating strict equivariance into these models remains a formidable challenge.

Current learning-based methods broadly fall into two main paradigms: sequential state prediction and joint trajectory generation. Sequential state prediction methods \cite{mpnet, mpnet2, mpnet3, wyh} iteratively predict future states. However, they not only suffer from compounding errors and slow sequential inference, but also fundamentally lack any consideration for spatial equivariance. Conversely, joint trajectory generation models, prominently diffusion-based frameworks \cite{ddpm, ddim, mpd, mpd2, edmp, diffusion-policy}, excel in feasible and multimodal trajectory generation but also face a geometric dilemma. Some variants\cite{eq-diffusion-policy, equibot, et-seed} achieve equivariance by employing specialized, mathematically heavy neural architectures, which critically bottleneck their real-time inference capabilities. Other methods \cite{mpd, mpd2} ignore geometric symmetries entirely, treating spatially translated or rotated tasks as entirely novel scenarios. This failure to capture underlying geometric symmetries results in brittle performance and a severe lack of spatial robustness.

Beyond the equivariance dilemma, generative models like diffusion also suffer from slow inference due to iterative denoising. Flow Matching (FM) \cite{fm, ot} addresses this speed bottleneck by constructing a smooth probability flow from initial noise to the target data distribution, enabling highly efficient few-step inference. However, existing FM-based planners \cite{flowmp, actionflow, rieflow} still suffer from three critical limitations. First, these FM methods still lack inherent spatial equivariance. Second, they naively employ standard Gaussian noise as the prior distribution \cite{fm}. As demonstrated by recent advances leveraging Gaussian Process \cite{hie-mpd} or task-informed priors \cite{data-dependent-coupling}, structured noise significantly enhances generative performance. Third, current generative planners \cite{mpd, mpd2, flowmp} ignore environmental geometries during training. Thus, their inherent capability to generate collision-free paths is restricted, almost entirely relying on computationally expensive cost-function gradient guidance during inference. 

To break this critical trade-off between geometric robustness and inference speed, we propose \textbf{BridgeFlow}, an environment-aware and strictly $\mathrm{SE(2)}$-equivariant Flow Matching framework for fast, robust motion planning. Our core contributions are summarized as follows:

\begin{itemize}
    \item We introduce a Brownian bridge informative prior paired with context-aware mini-batch optimal transport. By constructing the initial distribution as a noise-injected linear interpolation between endpoints and strictly confining OT couplings within identical task contexts, this synergy significantly straightens the flow field and drastically reduces transport costs.
    \item We leverage Classifier-Free Guidance \cite{cfg, zero-cfg, mpd-with-cfg} to explicitly embed environmental awareness. By conditioning the FM training on occupancy maps, we distill spatial constraints into the learned prior, eliminating the reliance on expensive inference-time optimization.
    \item We achieve strict $\mathrm{SE(2)}$-equivariance via a lightweight task-centric canonicalization module. This dynamically transforms inputs into a standardized frame and decanonicalizes the output vector field, achieving robust spatial generalization across continuous rigid body transformations with standard network architectures.
\end{itemize}

\section{Related Work}
\label{Sec.2}

\subsection{Generative Models for Motion Planning}
Learning-based motion planning broadly falls into two main paradigms: sequential state prediction and joint trajectory generation. Sequential models \cite{mpnet, mpnet2, mpnet3} formulate planning as an iterative decision-making process, which often suffers from compounding errors and slow sequential inference. Consequently, joint trajectory generation methods have gained prominence by directly modeling the feasible trajectory distribution. While foundational diffusion frameworks \cite{mpd, mpd2, edmp} capture robust multimodal priors, their iterative denoising fundamentally bottlenecks inference speed. To overcome this computational barrier, recent advances \cite{flowmp, actionflow, rieflow} have introduced Flow Matching (FM) \cite{fm, ot, curse-ot} to generative planning. FM constructs a smooth probability path between a tractable prior and the target data distribution, enabling highly efficient few-step inference via ordinary differential equation solvers. For instance, FlowMP \cite{flowmp} replaces the diffusion backbone of MPD \cite{mpd, mpd2} with FM to drastically accelerate trajectory sampling.

However, the naive application of FM to structured motion planning remains sub-optimal. First, standard FM often relies on task-agnostic Gaussian priors, which induces highly curved, complex vector fields, leading to inefficient transport. Second, existing generative planners handle environmental awareness sub-optimally. They either condition on heavy, high-dimensional sensory data\cite{scene-diffuser}, incurring severe computational overhead, or lack explicit environmental conditioning entirely \cite{mpd, mpd2, flowmp}. By ignoring environmental geometries during training, these foundational frameworks are forced to rely heavily on computationally expensive test-time cost-function gradients for obstacle avoidance. Addressing these intertwined inefficiencies through task-aligned priors, context-aware transport, and explicit environmental guidance forms the core motivation of our framework.

\subsection{Equivariant Learning}
In motion planning, solving structurally identical tasks that differ only in their global pose is frequent. To achieve this robust spatial generalization, generative planners must possess equivariance \cite{equivariance}, ensuring that the generated trajectory transforms deterministically alongside the input environment. Existing approaches to endow planners with equivariance typically rely on either data augmentation or specialized network architectures \cite{gecnn, se3-transformer, equivariant-learning, eq-diffusion-policy, equibot, et-seed}. 

Data augmentation requires massively inflated datasets and fails to guarantee strict mathematical equivariance. Conversely, specialized architectures entail complex custom operations, imposing significant computational overhead and complicating integration with standard generative pipelines like FM. To circumvent these bottlenecks, our framework proposes a lightweight task-centric canonicalization strategy. By pre-transforming inputs into a standardized coordinate frame, we achieve strict spatial generalization while maintaining full compatibility with standard neural architectures.

\section{Methodology}
\label{Sec.3}

In this section, we detail our fast and $\mathrm{SE(2)}$-equivariant generative motion planner, BridgeFlow. We begin by formally defining the generative motion planning objective (Sec. \ref{subsec:problem}). To address the transport inefficiencies of standard Flow Matching, we introduce a Brownian bridge informative prior (Sec. \ref{subsec:prior}) paired with context-aware mini-batch optimal transport (Sec. \ref{subsec:ot}). Next, we detail how environmental constraints are explicitly embedded into the generative process via Classifier-Free Guidance (Sec. \ref{subsec:cfg}). Finally, we present our lightweight task-centric canonicalization module (Sec. \ref{subsec:canonicalization}), which guarantees strict $\mathrm{SE(2)}$-equivariance across the entire framework.

\subsection{Problem Formulation}
\label{subsec:problem}



We formulate motion planning in a workspace $\mathbb{R}^d$ ($d \in \{2, 3\}$). A task is defined by the tuple $\mathcal{P} = (s, g, \mathcal{M})$, where $s, g \in \mathbb{R}^d$ are the start and goal positions, and $\mathcal{M}: \mathbb{R}^d \to \{0, 1\}$ is the environmental occupancy map. The objective is to learn a policy $\pi$ that generates a kinematically feasible, collision-free waypoint sequence $\mathbf{x} = [p_1, \dots, p_N] \in \mathbb{R}^{N \times d}$ conditioned on $\mathcal{P}$. 

To ensure spatial generalization, we formulate a strictly $\mathrm{SE(2)}$-equivariant motion planning problem. Consider a spatial rigid body transformation $T = (R, \mathbf{d}) \in SE(2)$, comprising a rotation $R \in SO(2)$ and a translation $\mathbf{d} \in \mathbb{R}^d$. The transformed task context is given by $\tilde{\mathcal{P}} = T(\mathcal{P}) = (Rs + \mathbf{d}, Rg + \mathbf{d}, T(\mathcal{M}))$. The generative policy $\pi$ must satisfy the exact equivariance constraint: 
$$ \pi(T(\mathcal{P})) = T(\pi(\mathcal{P})) $$

We realize this policy within a continuous-time Flow Matching framework, modeling trajectory generation as an ordinary differential equation over flow time $t \in [0, 1]$. Thus, the problem reduces to learning an $\mathrm{SE(2)}$-equivariant neural vector field $v_\theta(\mathbf{x}_t, t \mid \mathcal{P})$ that strictly satisfies the geometric pushforward rule:
$$ v_\theta(T(\mathbf{x}_t), t \mid T(\mathcal{P})) = R \cdot v_\theta(\mathbf{x}_t, t \mid \mathcal{P}) $$

This strict condition guarantees that the generative continuous dynamics remain geometrically consistent under any $\mathrm{SE(2)}$ transformation of the task.

\subsection{Brownian Bridge Informative Prior}
\label{subsec:prior}
The initial distribution $p_0$ is critical for the efficiency of Flow Matching. Standard FM typically employs a naive standard Gaussian distribution, which is entirely uninformative regarding the planning task, leading to complex vector fields and high transport costs. To address this, we propose a Brownian bridge informative prior that constructs $p_{\text{info}}$ based on the start and goal positions, aligning directly with robotic kinematic characteristics.

\begin{figure}[htbp]
    \centering
    \begin{subfigure}[b]{0.22\textwidth}
        \centering
        \includegraphics[width=\textwidth]{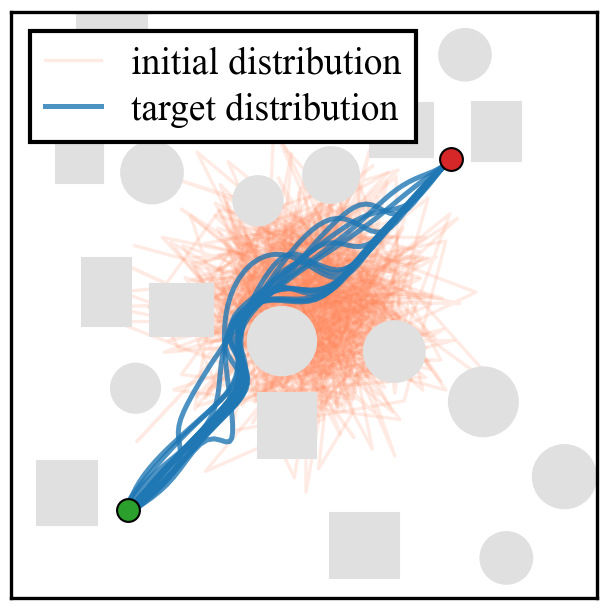}
        \caption{Gaussian Prior}
        \label{subfig:gaussian_prior}
    \end{subfigure}
    \hfill
    \begin{subfigure}[b]{0.22\textwidth}
        \centering
        \includegraphics[width=\textwidth]{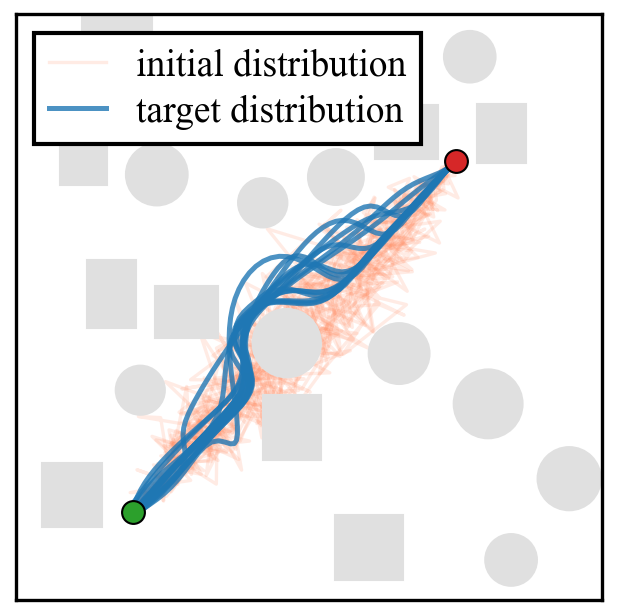}
        \caption{Brownian Bridge Prior}
        \label{subfig:bb_prior}
    \end{subfigure}
    \caption{Comparison of initial priors. The Brownian bridge anchors the endpoints, reducing required deformation.}
    \label{fig:priors}
\end{figure}

\subsubsection{Prior Design}
For a trajectory $\mathbf{x}_0$ consisting of $N$ waypoints at flow time $t=0$, let $\gamma_i = \frac{i-1}{N-1} \in [0, 1]$ be the normalized index of the $i$-th waypoint for $i \in\{1,\dots,N\}$. We construct a linear interpolation $\mu_{lin, i} = (1-\gamma_i)s + \gamma_i g$. The proposed prior $p_{\text{info}}$ is formulated point-wise as:
$$x_{0, i} = \mu_{lin, i} + \sigma_{\text{prior}} \cdot \sqrt{\gamma_i(1-\gamma_i)} \cdot \varepsilon, \quad \varepsilon \sim \mathcal{N}(0, I_d),$$
where $\sigma_{\text{prior}} \in \mathbb{R}^+$ scales the base noise. The \textit{bridge factor} $\sqrt{\gamma_i(1-\gamma_i)}$ analytically enforces the variance to strictly drop to zero at the start ($\gamma_1=0$) and goal ($\gamma_N=1$) while peaking in the middle. This mirrors the physical intuition of motion planning as fine trajectory adjustments bounded by fixed endpoints.

\subsubsection{Transport Cost Reduction}
To theoretically validate the advantage of $p_{\text{info}}$, we analyze the transport cost $\mathcal{C}(p_0, p_1)$, defined as the expected kinetic energy of the flow. Under the linear optimal transport probability path $\psi_t(\mathbf{x}) = (1-t)\mathbf{x}_0 + t\mathbf{x}_1$, the cost simplifies to the expected pairwise Euclidean distance: $\mathcal{C}(p_0, p_1) = \mathbb{E}_{\mathbf{x}_0, \mathbf{x}_1} [ \| \mathbf{x}_1 - \mathbf{x}_0 \|^2 $.



To quantify the cost reduction, we firstly decompose the expert trajectory of dimension $D = Nd$ as $\mathbf{x}_1 = \mu_{lin} + \delta$, where $\delta$ captures the spatial deformation for obstacle avoidance. For our Brownian bridge prior, we denote the step-wise noise scale as $\sigma_i = \sigma_{\text{prior}} \sqrt{\gamma_i(1-\gamma_i)}$, yielding an average variance $\bar{\sigma}^2 = \frac{1}{N}\sum_{i=1}^N \sigma_i^2$. We strictly bound the base noise such that $\bar{\sigma}^2 < 1$.

\begin{theorem}[Transport Cost Reduction]
    Given the trajectory decomposition $\mathbf{x}_1 = \mu_{lin} + \delta$, the transport cost of the informative prior ($\mathcal{C}_{\text{info}}$) is bounded below that of the standard Gaussian prior ($\mathcal{C}_{\text{std}}$). The exact reduction margin is:
    $$\mathcal{C}_{\text{std}} - \mathcal{C}_{\text{info}} = \| \mu_{lin} \|^2 + 2\langle \mu_{lin}, \mathbb{E}[\delta] \rangle + (1 - \bar{\sigma}^2)D.$$
\end{theorem}

\begin{proof}
For the standard prior $p_{\text{std}}$, tracking the expected distance with an independent $\mathbf{x}_0 \sim \mathcal{N}(0, I_D)$ yields:
\begin{align*}
    \mathcal{C}_{\text{std}} &= \mathbb{E} \left[ \| \mathbf{x}_1 \|^2 - 2\langle \mathbf{x}_1, \mathbf{x}_0 \rangle + \| \mathbf{x}_0 \|^2 \right]\\ 
    &= \| \mu_{lin} \|^2 + \mathbb{E}[\|\delta\|^2] + 2\langle \mu_{lin}, \bar{\delta} \rangle + D.
\end{align*}

For the informative prior $p_{\text{info}}$, substituting $\mathbf{x}_0 = \mu_{lin} + \epsilon$ (where point-wise $\epsilon_i \sim \mathcal{N}(0, \sigma_i^2 I_d)$) directly cancels $\mu_{lin}$:
\begin{align*}
    \mathcal{C}_{\text{info}} &= \mathbb{E} \left[ \| (\mu_{lin} + \delta) - (\mu_{lin} + \epsilon) \|^2 \right]\\ 
    &= \mathbb{E}[\|\delta\|^2] + \bar{\sigma}^2 D.
\end{align*}

Subtracting the two costs yields the exact margin:
$$\mathcal{C}_{\text{std}} - \mathcal{C}_{\text{info}} = \| \mu_{lin} \|^2 + 2\langle \mu_{lin}, \bar{\delta} \rangle + (1 - \bar{\sigma}^2)D.$$
Considering symmetrically distributed obstacles across the dataset, the expected deformation neutralizes ($\bar{\delta} \to \mathbf{0}$). Since the linear energy $\| \mu_{lin} \|^2 > 0$ and $(1 - \bar{\sigma}^2)D > 0$, the exact margin is strictly positive, proving $\mathcal{C}_{\text{info}} < \mathcal{C}_{\text{std}}$.
\end{proof}

\begin{remark}
    Minimizing $\sigma_{\text{prior}}$ reduces $\mathcal{C}_{\text{info}}$ but risks collapsing the prior into a more deterministic state, stripping the model's multimodal capacity. Thus, $\sigma_{\text{prior}}$ is a critical hyperparameter balancing transport efficiency and generative diversity.
\end{remark}

\subsubsection{SE(2)-Equivariance of the Brownian Bridge Prior}
The $p_{\text{info}}$ inherently exhibits $\mathrm{SE(2)}$-equivariance, seamlessly enabling the task-centric canonicalization. 

\begin{proposition}\label{prior eq}
    For any trajectory sample $\mathbf{x}_0 \sim p_0(\cdot \mid \mathcal{P})$ and transformation $T \in SE(2)$, the transformed sample exactly matches the distribution of the prior conditioned on the transformed task: $T(\mathbf{x}_0) \sim p_0(\cdot \mid T(\mathcal{P}))$.
\end{proposition}

\begin{proof}
    First, the interpolation mean strictly commutes with $T$: $T(\mu_{lin}(\gamma)) = R\mu_{lin}(\gamma) + \mathbf{d}$. Then, applying $T$ to a waypoint $x_0(\gamma)$ sampled from the $p_{\text{info}}$ yields:
    $$T(x_0(\gamma)) = R(\mu_{lin}(\gamma) + \sigma(\gamma)\epsilon) + \mathbf{d} = T(\mu_{lin}(\gamma)) + \sigma(\gamma)R\varepsilon.$$
    Since the base noise $\varepsilon$ is isotropic and $R$ is orthogonal, the rotated noise $R\varepsilon$ retains the covariance $R I_d R^T = I_d$. Thus, it is identical to unrotated noise. Consequently, $T(x_0(\gamma)) = T(\mu_{lin}(\gamma)) + \sigma(\gamma)\varepsilon$, proving strict $\mathrm{SE(2)}$-equivariance.
\end{proof}

\begin{figure*}[htbp]
    \centering 
    \includegraphics[width=0.9\linewidth]{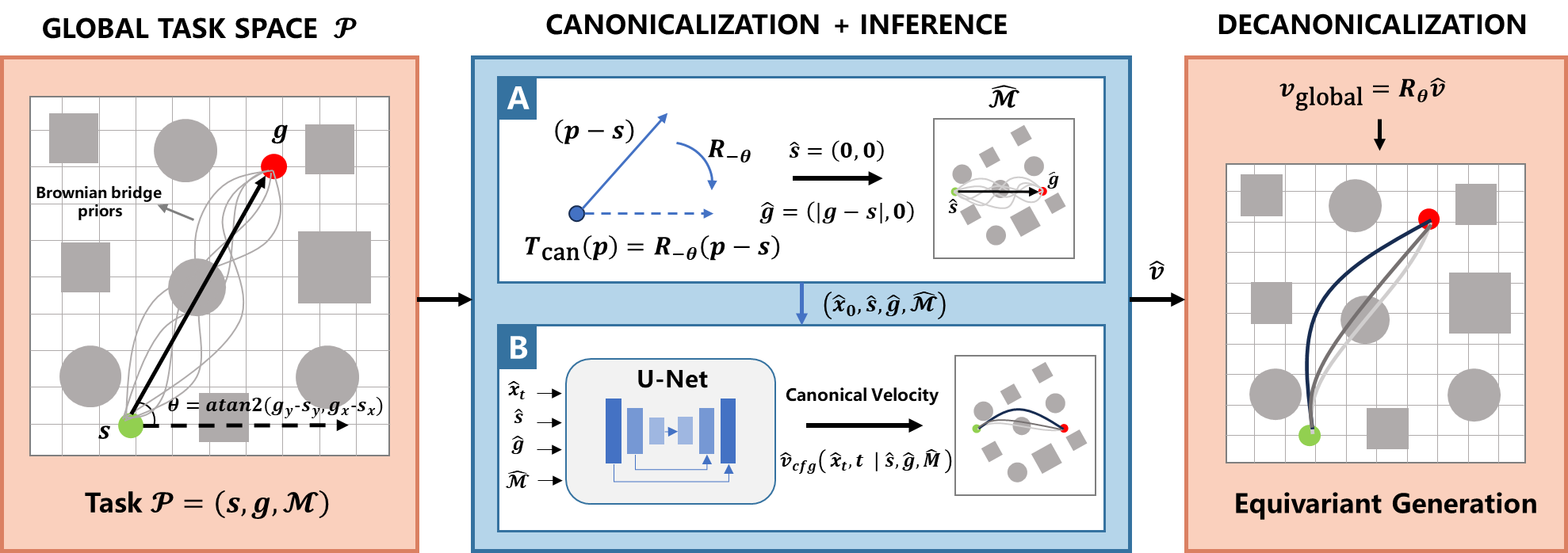}
    \caption{Task-Centric Canonicalization Framework. The system achieves strict $\mathrm{SE(2)}$-equivariance by projecting global tasks into a unified canonical frame. The input is normalized via $T_{can}(p) = R_{-\theta}(p - s)$, anchoring $\hat{s}$ at the origin and aligning $\hat{g}$ with the x-axis. A standard U-Net then infers the canonical vector field $\hat{v}$ conditioned on the aligned map $\widehat{\mathcal{M}}$. The canonical velocity is decanonicalized via $v_{global} = R_\theta \hat{v}$, ensuring generated trajectories transform deterministically with any rigid body transformation of the workspace.}
    \label{fig:SE(2)-Framework}
\end{figure*}

\subsection{Context-Aware Mini-Batch Optimal Transport}
\label{subsec:ot}

Standard mini-batch optimal transport (OT) blindly pairs initial noise $\mathbf{x}_0$ and expert trajectories $\mathbf{x}_1$ across an entire training batch $\mathcal{B}$, fundamentally ignoring the underlying task condition $c = (s, g, \mathcal{M})$ in our problem. Consequently, it inadvertently cross-pairs a prior sampled for one specific context with an expert trajectory belonging to a completely different context. Because the parameterized vector field $v_\theta$ is explicitly conditioned on $c$, such cross-context pairing forces the network to learn physically contradictory mappings. This ill-posed regression target inherently destabilizes training and collapses the learned conditional flow field.

To rectify this, we propose context-aware mini-batch OT (Algorithm \ref{alg:context_ot}). Instead of a global batch assignment, we dynamically partition the mini-batch into context-specific disjoint subsets $\mathcal{B} = \bigcup_{n=1}^N \mathcal{B}_{c_n}$, where each subset $\mathcal{B}_{c_n}$ exclusively contains trajectories sharing the exact same task condition $c_n$. We then solve the linear assignment strictly within each local subset  using the Hungarian algorithm:
$$ \sigma_n^* = \arg\min_{\sigma} \sum_{k=1}^{|\mathcal{B}_{c_n}|} \left\| \mathcal{X}_{0, c_n}^{(k)} - \mathcal{X}_{1, c_n}^{(\sigma(k))} \right\|^2, \quad \forall n \in \{1, \dots, N\} $$

By explicitly confining the OT coupling to trajectories sharing identical boundary conditions and obstacle layouts, we completely eliminate physically impossible cross-context pairings. This formulation guarantees the learning of smooth, task-consistent conditional flow fields, thereby further straightening transport paths and accelerating both training and inference processes.

\begin{algorithm}[htbp]
\caption{Context-Aware Mini-Batch Optimal Transport}
\label{alg:context_ot}
\begin{algorithmic}[1]
\STATE \textbf{Input:} Mini-batch of expert trajectories and task contexts $\mathcal{B} = \{(\mathbf{x}_1^{(i)}, c^{(i)})\}_{i=1}^B$ where $c = (s, g, \mathcal{M})$, trainable vector field policy $v_\theta$
\STATE \textbf{Output:} Context-aligned FM regression loss $\mathcal{L}_{FM}(\theta)$
\STATE Initialize accumulated loss $\mathcal{L}_{FM} \leftarrow 0$
\STATE Extract unique task contexts $\mathcal{C} = \{c_1, \dots, c_N\}$ from $\mathcal{B}$
\FOR{each context $c_n \in \mathcal{C}$}
    \STATE Extract $K_n$ experts: $\mathcal{X}_1^{(n)} \!=\! \{\mathbf{x}_1^{(i)} \mid c^{(i)} \!=\! c_n\}$
    \STATE Sample Brownian Bridge noise: $\mathcal{X}_0^{(n)} \sim p_0(\cdot \mid c_n)$
    \STATE Compute cost matrix: $\mathbf{C}_{i,j} = \|\mathcal{X}_0^{(n)}[i] - \mathcal{X}_1^{(n)}[j]\|^2$
    \STATE Solve assignment: $\sigma^* = \arg\min_{\sigma} \sum_k \mathbf{C}_{k, \sigma(k)}$
    \STATE Sample batched flow times: $\mathbf{t} \sim \mathcal{U}(0, 1)^{K_n}$
    \STATE Interpolate states: $\mathbf{X}_{\mathbf{t}} = (1-\mathbf{t})\mathcal{X}_0^{(n)} + \mathbf{t}\mathcal{X}_1^{(n)}[\sigma^*]$
    \STATE Target vector fields: $\mathbf{U}_{\mathbf{t}} = \mathcal{X}_1^{(n)}[\sigma^*] - \mathcal{X}_0^{(n)}$
    \STATE $\mathcal{L}_{batch} \mathrel{+=} \sum_k \| v_\theta(\mathbf{X}_{\mathbf{t}}[k], \mathbf{t}[k] \mid c_n) - \mathbf{U}_{\mathbf{t}}[k] \|^2$
\ENDFOR
\STATE \textbf{return} $\mathcal{L}_{FM} / B$ \COMMENT{Gradient objective for policy update}
\end{algorithmic}
\end{algorithm}

\subsection{Environment-Aware Classifier-Free Guidance}
\label{subsec:cfg}
To directly embed environment awareness and eliminate the reliance on computationally expensive optimizations, we explicitly integrate the environment occupancy map $\mathcal{M}$ into the Flow Matching framework via Classifier-Free Guidance.

During training, we randomly mask the environment condition $\mathcal{M}$ with a learnable null token $\emptyset$ at a fixed dropout probability $p_{\text{drop}}$. This forces the network to jointly learn an environment-aware vector field, denoted concisely as $v^\mathcal{M}_\theta \triangleq v_\theta(\mathbf{x}_t, t \mid s, g, \mathcal{M})$, and an unconditional field $v^\emptyset_\theta \triangleq v_\theta(\mathbf{x}_t, t \mid s, g, \emptyset)$ that relies solely on task endpoints.

During inference, we extrapolate the flow in the direction of the environmental condition. The final guided vector field $v_{CFG}$ used to update the trajectory ODE is formulated as:
$$v_{CFG}(\mathbf{x}_t, t) = v^\emptyset_\theta + \omega \cdot \left( v^\mathcal{M}_\theta - v^\emptyset_\theta \right),$$
where the guidance scale $\omega \ge 1$ acts as a crucial control lever, seamlessly balancing collision-free performance with the preservation of multimodal generation. 

By distilling environment awareness directly into the learned vector field, this explicit integration completely eliminates the need for iterative, cost-based gradient guidance during sampling, fully preserving the inherent low-latency advantage of Flow Matching.

\subsection{Task-Centric Canonicalization}
\label{subsec:canonicalization}
To achieve $\mathrm{SE(2)}$-equivariance without incurring the computational overhead of specialized network architectures, we propose a lightweight task-centric canonicalization module \ref{fig:SE(2)-Framework} that deterministically transforms all spatial inputs into a standardized canonical frame prior to network inference.

\subsubsection{SE(2)-Equivariance Definition}
Let $T \in SE(2)$ be parameterized by a rotation matrix $R_\alpha \in SO(2)$ and a translation vector $\mathbf{d} \in \mathbb{R}^2$. The group action applies to points as $T(p) = R_\alpha p + \mathbf{d}$, and applies pointwise to a trajectory waypoint sequence $\mathbf{x}$. Occupancy maps transform via the pullback operator $(T(\mathcal{M}))(p) = \mathcal{M}(T^{-1}(p))$. Because continuous vector fields represent time derivatives (velocities) of trajectories, they are strictly translation-invariant and solely rotate: $T(v(\mathbf{x})) = R_\alpha v(\mathbf{x})$. Consequently, our generative policy $v_\theta$ is formally $\mathrm{SE(2)}$-equivariant if, for any valid $T \in SE(2)$, it satisfies:
$$v_\theta(T(\mathbf{x}_t), t \mid T(s), T(g), T(\mathcal{M})) = R_\alpha \, v_\theta(\mathbf{x}_t, t \mid s, g, \mathcal{M}).$$

\subsubsection{Equivariant Framework Construction}
Let $\theta = \text{atan2}(g_y - s_y, g_x - s_x)$ denote the global heading angle. We define a canonicalization operator $T_{can}(p) = R_{-\theta}(p - s)$ that translates the start $s$ to the origin and aligns the start-goal vector with the positive x-axis. The equivariant generation executes a \textit{Canonicalize-Infer-Decanonicalize} pipeline. First, inputs are mapped into the canonical frame: $\hat{\mathbf{x}}_t = T_{can}(\mathbf{x}_t)$ and $\hat{\mathcal{M}}(p) = \mathcal{M}(T_{can}^{-1}(p))$, mapping endpoints deterministically to $\hat{s} = \mathbf{0}$ and $\hat{g} = (\|g-s\|, 0)$. Second, the network predicts the flow entirely within the canonical space: $\hat{v} = v_\theta(\hat{\mathbf{x}}_t, t \mid \hat{s}, \hat{g}, \hat{\mathcal{M}})$. Finally, the output velocities are mapped back to the global frame via the inverse rotation: $v_{global} = R_\theta \hat{v}$.

\begin{theorem}[$\mathrm{SE(2)}$-Equivariance]
    The proposed canonicalization framework guarantees strict $\mathrm{SE(2)}$-equivariance.
\end{theorem}

\begin{proof}
    Consider a planning task $\mathcal{P} = (s, g, \mathcal{M})$ and an arbitrary rigid transformation $T_{world} = (R_\alpha, \mathbf{d})$. The transformed task is denoted as $\tilde{\mathcal{P}} = T_{world}(\mathcal{P}) = (\tilde{s}, \tilde{g}, \tilde{\mathcal{M}})$. As established in Proposition \ref{prior eq}, the initial noise sampled for the transformed task corresponds exactly to the transformed noise of the original task: $\tilde{\mathbf{x}}_0 = T_{world}(\mathbf{x}_0)$. 

    Because the start-goal vector rotates by $\alpha$, the updated global heading angle is $\tilde{\theta} = \theta + \alpha$. The canonical transformation operator for the new task becomes $\tilde{T}_{can}(p) = R_{-\tilde{\theta}}(p - \tilde{s})$. Applying this operator to the transformed trajectory yields:
    $$ \begin{aligned} \tilde{T}_{can}(\tilde{\mathbf{x}}_t) &= R_{-(\theta + \alpha)}\left( (R_\alpha \mathbf{x}_t + \mathbf{d}) - (R_\alpha s + \mathbf{d}) \right)\\ &= R_{-\theta} R_{-\alpha} R_\alpha (\mathbf{x}_t - s) = \hat{\mathbf{x}}_t. \end{aligned}$$
    This rigorous cancellation demonstrates that the canonical trajectory remains absolutely invariant regardless of global shifts. By identical logic, applying $\tilde{T}_{can}$ to the transformed map recovers the original canonical map: $\tilde{T}_{can}(\tilde{\mathcal{M}}) = \hat{\mathcal{M}}$.

    Given mathematically identical canonical inputs $(\hat{\mathbf{x}}_t, t \mid \hat{s}, \hat{g}, \hat{\mathcal{M}})$, the neural network infers the exact same canonical vector field: $\hat{v}$. Finally, the decanonicalization step computes the global vector field for the transformed task:
    $$\tilde{v}_{global} = R_{\tilde{\theta}}\hat{v} = R_{\theta + \alpha}\hat{v} = R_\alpha (R_\theta \hat{v}) = R_\alpha v_{global}.$$
    This exact match with the vector pushforward rule $\tilde{v}_{global} = T_{world}(v_{global})$ definitively proves that the framework is strictly $\mathrm{SE(2)}$-equivariant. 
\end{proof}

\section{Experiments}
\label{Sec.4}
The overarching goal of our empirical evaluation is to validate the efficiency, generalization capability, and geometric robustness of the proposed BridgeFlow framework. Specifically, our experiments are designed to answer the following core Research Questions (RQs):
\begin{itemize}
    \item \textbf{RQ1 (Efficiency \& Quality):} Does the Brownian bridge informative prior and context-aware mini-batch OT significantly accelerate the inference speed of Flow Matching while maintaining trajectory quality?
    \item \textbf{RQ2 (Environmental Generalization):} Does the environment-aware Classifier-Free Guidance (CFG) enable the model to successfully adapt to novel, unseen obstacle configurations without heavy reliance on inference-time optimization?
    \item \textbf{RQ3 (Geometric Robustness):} Does the task-centric canonicalization effectively guarantee $\mathrm{SE(2)}$-equivariance, enabling generalization to spatially transformed tasks without augmented training data?
\end{itemize}

\begin{figure}[htbp]
    \centering
    \begin{subfigure}[b]{0.22\textwidth}
        \centering
        \includegraphics[width=\textwidth]{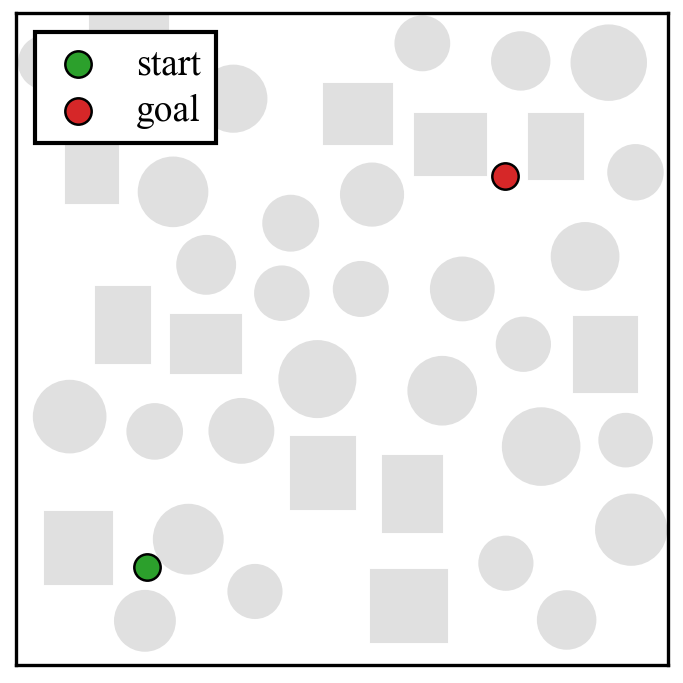}
        \caption{PointMass2D Dense}
        \label{subfig:2d_scene}
    \end{subfigure}
    \hfill
    \begin{subfigure}[b]{0.22\textwidth}
        \centering
        \includegraphics[width=\textwidth]{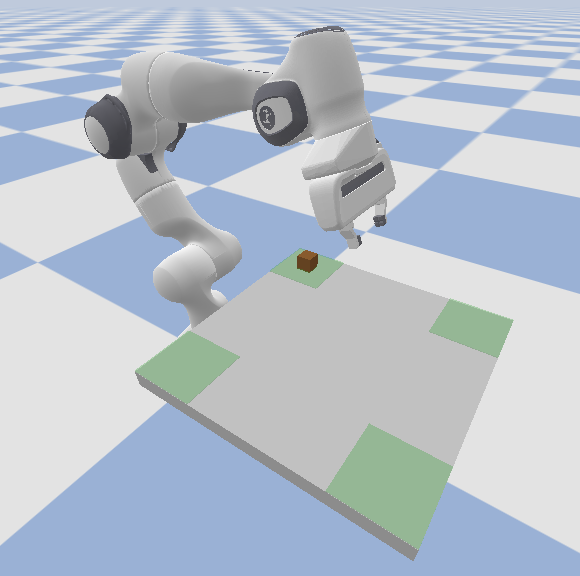}
        \caption{7-DoF Franka Panda}
        \label{subfig:pybullet_scene}
    \end{subfigure}
    \caption{Experimental setups for evaluating generative motion planning. (a) An example of the 2D planar navigation environment heavily cluttered with obstacles. (b) The 3D manipulation workspace in PyBullet. The symmetric square tabletop features four distinct target regions (A, B, C, and D) marked in green at its corners, specifically designed for evaluating $\mathrm{SE(2)}$ spatial generalization.}
    \label{fig:environments}
\end{figure}

\subsection{Experimental Setup}

\subsubsection{Environments and Data Collection}
To comprehensively evaluate our method, we conduct experiments across two distinct domains. We utilize an RRT-Connect planner followed by B-spline smoothing to generate collision-free, kinematically smooth expert trajectories for training.

\textbf{PointMass2D Dense:} In this foundational planar navigation benchmark, a point mass must find a smooth path through densely cluttered obstacles. We generated 200 distinct environments, explicitly partitioned into 150 for training and 50 strictly unseen environments for testing environmental generalization. For each environment, we sampled 500 valid start-goal pairs with 20 feasible multimodal trajectories per pair. By leveraging the time-reversibility of kinematic paths (swapping starts and goals), we cost-effectively augmented the dataset to 1,000 directed tasks per environment without incurring additional simulation overhead.

\textbf{7-DoF Franka Emika Panda:} To evaluate $\mathrm{SE(2)}$-equivariance in high-dimensional 3D manipulation, we simulate a robotic arm performing pick-and-place operations. The symmetric tabletop workspace features four distinct planar regions (A, B, C, and D) located at the corners of a square. We deliberately restrict the training distribution to contain \textit{only} trajectories originating from region A (i.e., $A \to B, C, D$). During testing, the model executes paths between unseen regions (e.g., $B \to C$, $D \to B$). Because these testing tasks represent pure $\mathrm{SE(2)}$ rigid body transformations of the training distribution, this setup perfectly isolates the evaluation of geometric generalization.

\subsubsection{Baselines}
We compare BridgeFlow against two state-of-the-art generative planning frameworks. \textbf{MPD (Motion Planning Diffusion)} is a representative diffusion-based planner that learns strong trajectory priors but suffers from iterative denoising latency and lacks explicit environmental guidance during training. \textbf{FlowMP} replaces the diffusion backbone with standard Flow Matching to accelerate inference, yet it inherits a naive Gaussian prior.

\subsubsection{Evaluation Metrics}

 To fully leverage the parallel computing capabilities of GPUs, we sample $K=40$ candidate trajectories for each testing context. Success Rate (SR) is the percentage of planning contexts where at least one generated trajectory successfully reaches the goal without collisions. Valid Rate (VR) is the proportion of completely valid paths among all candidates and serves as a proxy for the quality of the generative prior. Collision Intensity is the ratio of waypoints in collision across the entire set of candidates to quantify environmental awareness. Average Length and Diversity represent the mean path length and the spatial variance among successful trajectories, respectively. Inference Time is the average computational latency required to generate the full candidate set.

\begin{figure*}[htbp]
    \centering 
    \includegraphics[width=0.9\linewidth]{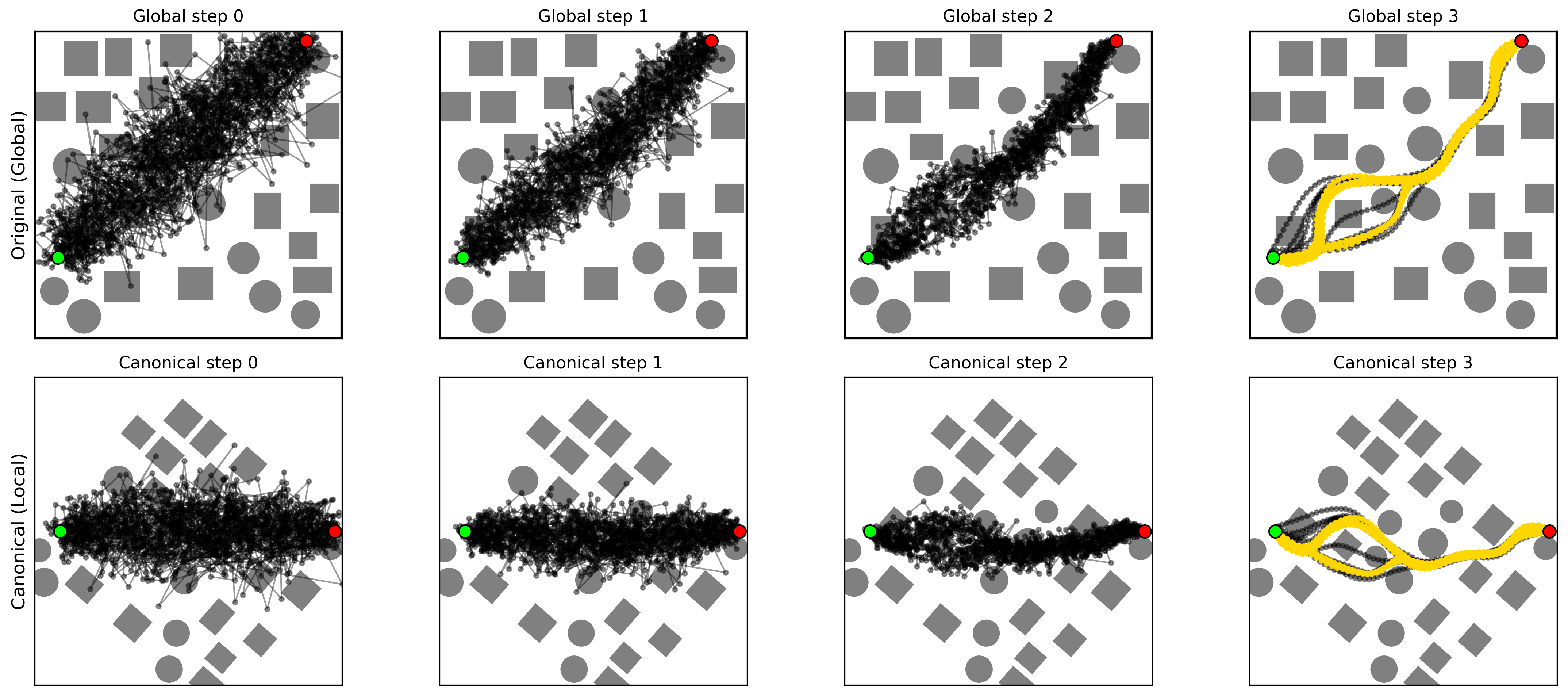}
    \caption{Visualization of the 3-step Flow Matching generative process in the PointMass2D Dense environment. The top row illustrates the trajectory evolution in the original, uncanonicalized global workspace. The bottom row demonstrates the exact same generation process within our task-centric canonicalization frame. Yellow curves represent feasible, collision-free trajectories, while black curves indicate infeasible paths that collide with obstacles.}
    \label{fig:2d_experiment}
\end{figure*}

\begin{table*}[t]
\centering
\caption{Comprehensive Ablation Study and Baseline Comparison in PointMass2D Dense Environment}
\label{tab:pointmass_ablation}
\resizebox{\textwidth}{!}{%
\begin{tabular}{lc | cccccc | cccccc}
\toprule
\multirow{2}{*}{\textbf{Method}} & \multirow{2}{*}{\textbf{Steps}} & \multicolumn{6}{c|}{\textbf{Training Distribution (Dataset)}} & \multicolumn{6}{c}{\textbf{Unseen Distribution (Random Context)}} \\
\cmidrule(lr){3-8} \cmidrule(lr){9-14}
& & SR (\%) & VR (\%) & Intensity $\downarrow$ & Avg Length $\downarrow$ & Diversity $\uparrow$ & Time (s) $\downarrow$ & SR (\%) & VR (\%) & Intensity $\downarrow$ & Avg Length $\downarrow$ & Diversity $\uparrow$ & Time (s) $\downarrow$ \\
\midrule
MPD & 5 (DDIM) & 98 & 39.80 & 10.87 & 6.6820 & 24.6527 & 0.0410 & 71 & 17.95 & 15.28 & 7.6305 & 27.1953 & 0.0411 \\
MPD & 10 (DDIM) & 98 & 42.95 & 9.66 & 6.6311 & 30.1687 & 0.0769 & 81 & 19.55 & 14.68 & 7.8246 & 33.8155 & 0.0765 \\
MPD & 15 (DDIM) & 98 & 42.48 & 9.67 & \textbf{6.6226} & 32.8527 & 0.1128 & 83 & 19.05 & 14.63 & 7.9008 & 36.9619 & 0.1116 \\
\midrule
MPD Prior & 5 (DDIM) & 98 & 39.45 & 10.94 & 6.6996 & 24.6637 & 0.0141 & 71 & 17.48 & 16.15 & 7.7226 & 27.3323 & 0.0144 \\
MPD Prior & 10 (DDIM) & 98 & 42.70 & 9.73 & 6.6430 & 30.1888 & 0.0255 & 79 & 18.27 & 15.51 & 7.7856 & 34.0778 & 0.0256 \\
MPD Prior & 15 (DDIM) & 98 & 42.25 & 9.74 & 6.6345 & \textbf{32.8816} & 0.0382 & 78 & 18.20 & 15.43 & 7.9227 & \textbf{37.3093} & 0.0380 \\
\midrule
FlowMP & 2 & 96 & 56.75 & 6.26 & 6.7602 & 4.7424 & 0.0072 & 66 & 27.60 & 11.96 & 7.7251 & 4.9368 & 0.0075 \\
FlowMP & 3 & 99 & 61.05 & 5.04 & 6.7980 & 9.2915 & 0.0098 & 74 & 29.60 & 10.88 & 7.8051 & 9.6412 & 0.0096 \\
FlowMP & 4 & 100 & 61.55 & 4.65 & 6.8545 & 13.2270 & 0.0123 & 81 & 29.88 & 10.47 & 7.8658 & 13.8280 & 0.0120 \\
FlowMP & 5 & 100 & 61.15 & 4.56 & 6.8496 & 16.3040 & 0.0138 & 80 & 29.93 & 10.17 & 7.8317 & 17.0199 & 0.0140 \\
\midrule
FlowMP + OT & 2 & 99 & 53.62 & 6.28 & 6.8530 & 17.4525 & 0.0075 & 79 & 26.80 & 12.15 & 7.9296 & 17.7293 & 0.0076 \\
FlowMP + OT & 3 & 100 & 57.40 & 5.24 & 6.9228 & 21.0322 & 0.0097 & 80 & 28.98 & 11.17 & 8.0438 & 21.4398 & 0.0099 \\
FlowMP + OT & 4 & 100 & 58.17 & 4.90 & 6.9196 & 23.4120 & 0.0120 & 80 & 29.77 & 10.78 & 7.9794 & 24.0079 & 0.0120 \\
FlowMP + OT & 5 & 100 & 58.38 & 4.79 & 6.9301 & 25.0982 & 0.0139 & 84 & 29.88 & 10.61 & 8.0604 & 25.8435 & 0.0149 \\
\midrule
FlowMP + BB Prior & 2 & 98 & 58.40 & 5.95 & 6.6487 & 4.6632 & 0.0074 & 67 & 26.23 & 11.65 & 7.6385 & 4.7095 & \textbf{0.0072} \\
FlowMP + BB Prior & 3 & 100 & 64.65 & 4.42 & 6.6821 & 8.9906 & 0.0100 & 78 & 28.55 & 10.64 & 7.6823 & 9.1608 & 0.0096 \\
FlowMP + BB Prior & 4 & 100 & 66.53 & 3.93 & 6.6900 & 12.7877 & 0.0115 & 81 & 29.45 & 10.36 & 7.6567 & 13.0263 & 0.0119 \\
FlowMP + BB Prior & 5 & 100 & 66.12 & 3.70 & 6.7051 & 15.9542 & 0.0139 & 82 & 29.57 & 10.22 & \textbf{7.6132} & 16.2149 & 0.0140 \\
\midrule
\textbf{BridgeFlow (Ours)} & 2 & 100 & 70.25 & 4.07 & 6.6617 & 8.6599 & \textbf{0.0071} & 82 & 32.95 & 10.16 & 7.6678 & 8.1654 & 0.0075 \\
\textbf{BridgeFlow (Ours)} & 3 & 100 & 77.62 & 2.43 & 6.6903 & 13.0911 & 0.0096 & 88 & 37.05 & 8.38 & 7.8840 & 12.7498 & 0.0094 \\
\textbf{BridgeFlow (Ours)} & 4 & 100 & 80.03 & 1.89 & 6.7442 & 16.6112 & 0.0117 & \textbf{90} & 38.75 & 7.75 & 7.9164 & 16.2652 & 0.0117 \\
\textbf{BridgeFlow (Ours)} & 5 & \textbf{100} & \textbf{80.95} & \textbf{1.67} & 6.7762 & 19.5054 & 0.0133 & \textbf{90} & \textbf{39.40} & \textbf{7.50} & 7.9624 & 19.0508 & 0.0140 \\
\bottomrule
\end{tabular}%
}
\end{table*}

\subsection{Results in PointMass2D Dense Environment}
We systematically evaluate our framework in the PointMass2D dense environment to answer three core Research Questions (RQs) regarding generative quality, environmental generalization, and geometric robustness. Furthermore, we detail a critical architectural insight regarding boundary satisfaction. Quantitative results are summarized in Table \ref{tab:pointmass_ablation} and Table \ref{tab:se2_2d}.

\subsubsection{RQ1: Generative Quality and Inference Efficiency}
Table \ref{tab:pointmass_ablation} demonstrates that BridgeFlow drastically outperforms standard diffusion (MPD) and naive Flow Matching (FlowMP) in both safety and speed. Operating at 5 inference steps, BridgeFlow achieves an exceptional Valid Rate (VR) of $80.95\%$ and a minimal Collision Intensity of $1.67$, effectively doubling the safety of MPD (VR $42.48\%$, Intensity $9.67$). Remarkably, BridgeFlow requires only \textbf{2 steps} ($0.0071$s) to reach a $100\%$ SR and $70.25\%$ VR. To achieve comparable performance, MPD requires 15 DDIM steps ($0.1128$s), highlighting a remarkable \textbf{$15\times$ inference speedup}. This performance stems from the synergy between the Brownian bridge prior and context-aware mini-batch OT. Ablation studies reveal that standard OT alone degrades VR to $58.38\%$ due to cross-context matching conflicts. However, when combined with the Brownian bridge prior, VR surges to $80.95\%$, proving that this synergy reduces transport cost and straightens the flow field.

\subsubsection{RQ2: Generalization to Unseen Contexts}
To test environmental generalization, we evaluated all models on Random Contexts featuring completely novel obstacle layouts and unseen start-goal pairs. For a fair comparison, we integrated our environment-aware CFG mechanism into all baseline models. Despite the distribution shift, BridgeFlow demonstrates vastly superior robustness, maintaining a VR of $39.4\%$ and an Intensity of $7.5$. In contrast, the CFG-enhanced MPD collapses to a VR of $19.05\%$ and an Intensity of $14.63$. This substantial margin reveals a critical insight: environmental guidance via CFG is most effective when applied to the structured, straightened vector fields provided by our framework.

\begin{table*}[htbp]
\centering
\caption{Zero-Shot Geometric Generalization on 7-DoF Franka Manipulation Task}
\label{tab:franka_se2}
\resizebox{\textwidth}{!}{%
\begin{tabular}{l | cccc | cccc}
\toprule
\multirow{2}{*}{\textbf{Method}} & \multicolumn{4}{c|}{\textbf{In-Distribution ($A \to B/C/D$)}} & \multicolumn{4}{c}{\textbf{Unseen SE(2) Transfers (other regions)}} \\
\cmidrule(lr){2-5} \cmidrule(lr){6-9}
& VR (\%) $\uparrow$ & Avg Length $\downarrow$ & Diversity $\uparrow$ & Time (s) $\downarrow$ & VR (\%) $\uparrow$ & Avg Length $\downarrow$ & Diversity $\uparrow$ & Time (s) $\downarrow$ \\
\midrule
MPD (15 Step) & 43.0 & 4.82 & 15.21 & 0.152 & 19.0 & 5.10 & 18.54 & 0.155 \\
FlowMP (5 Step) & 65.0 & 4.51 & 8.52 & 0.038 & 21.0 & 4.95 & 9.15 & 0.039 \\
\textbf{BridgeFlow (5 Step)} & \textbf{88.0} & 4.23 & 12.45 & 0.045 & \textbf{80.0} & 4.30 & 11.82 & 0.046 \\
\bottomrule
\end{tabular}%
}
\end{table*}

\subsubsection{RQ3: Geometric Robustness and SE(2)-Equivariance}
To evaluate geometric robustness, we subjected the models to continuous spatial transformations ($T \in SE(2)$) on the original training tasks (Table \ref{tab:se2_2d}). Baselines lacking structural equivariance (MPD, FlowMP) suffer catastrophic failures: SRs plummet to $40\%\text{--}52\%$, VRs drop below $6\%$, and Collision Intensities surge to nearly $25$. This confirms that standard generative planners merely memorize absolute coordinate distributions, entirely losing spatial reasoning. In contrast, BridgeFlow maintains commanding generative performance (SR $96\%$, VR $74.85\%$, Intensity $5.28$), proving that our task-centric canonicalization module successfully factors out global symmetries for zero-shot generalization.


\begin{table}[htbp]
\centering
\caption{Evaluation on Random SE(2) Transformed Tasks}
\label{tab:se2_2d}
\resizebox{\columnwidth}{!}{%
\begin{tabular}{lcccccc}
\toprule
\textbf{Method} & \textbf{SR (\%)} & \textbf{VR (\%)} & \textbf{Int. $\downarrow$} & \textbf{Len. $\downarrow$} & \textbf{Div. $\uparrow$} & \textbf{Time (s)} \\
\midrule
MPD (5 Step)  & 44 & 5.75 & 23.75 & 6.3229 & 24.1820 & 0.0296 \\
MPD (10 Step) & 50 & 5.75 & 23.56 & 6.0229 & 28.3098 & 0.0592 \\
MPD (15 Step) & 52 & 5.65 & 23.55 & 6.0239 & 31.3268 & 0.0880 \\
\midrule
FlowMP (2 Step) & 20 & 6.20 & 25.13 & 5.9499 & 2.7843 & 0.0117 \\
FlowMP (3 Step) & 32 & 5.30 & 25.15 & 6.1966 & 6.6363 & 0.0176 \\
FlowMP (4 Step) & 30 & 4.65 & 25.16 & 6.1787 & 10.1282 & 0.0232 \\
FlowMP (5 Step) & 40 & 4.65 & 24.98 & 5.8222 & 13.0033 & 0.0289 \\
\midrule
BridgeFlow (2 Step) & 88 & 50.65 & 5.80 & 6.5569 & 2.9351 & 0.0129 \\
BridgeFlow (3 Step) & 94 & 54.05 & 5.60 & 6.7047 & 6.6146 & 0.0185 \\
BridgeFlow (4 Step) & 96 & 63.45 & 5.54 & 6.8044 & 10.4302 & 0.0244 \\
\textbf{BridgeFlow (5 Step)} & \textbf{96} & \textbf{74.85} & \textbf{5.28} & 6.8583 & 13.1781 & 0.0302 \\
\bottomrule
\end{tabular}%
}
\end{table}

\subsubsection{Insight: Natural Boundary Satisfaction without In-Painting}
Beyond quantitative metrics, we observed a architectural advantage: Unlike baselines requiring heuristic ``trajectory in-painting'' (overwriting endpoints at each step), BridgeFlow naturally satisfies boundary conditions. Because our Brownian bridge prior analytically enforces zero variance at the boundaries, throughout training, the neural network naturally learns that the vector field at these fixed endpoints must be exactly zero ($v_\theta = 0$). This enables the ODE solver to integrate trajectories autonomously, preserving mathematical continuity without manual intervention.

\subsection{Results on 7-DoF Franka Manipulation}
To validate the scalability and geometric robustness of our framework within high-dimensional, complex robotic systems, we evaluate the models on a 7-DoF Franka Emika Panda manipulation task.  The experimental workspace is divided into four symmetric tabletop quadrants (A, B, C, and D). Crucially, the training dataset exclusively comprises trajectories originating from quadrant A (e.g., $A \to B, C, D$). During evaluation, we challenge the models with Out-of-Distribution (OOD) tasks originating from the other quadrants (e.g., $B \to C, D \to B$). These OOD tasks represent strict $\mathrm{SE(2)}$ spatial transformations of the training distribution. Quantitative results are detailed in Table \ref{tab:franka_se2}.

\subsubsection{RQ1: High-Dimensional Generative Performance and Efficiency}
Even within the in-distribution tasks, BridgeFlow establishes a commanding advantage. In the complex 7-DoF joint space, it achieves an exceptional Valid Rate (VR) of $88.0\%$, significantly outperforming both the diffusion-based MPD ($43.0\%$) and the naive FlowMP baseline ($65.0\%$). Furthermore, BridgeFlow consistently generates the most kinematically efficient paths, demonstrating that our context-aware mini-batch OT successfully straightens the probability flow field even when scaled to high-dimensional continuous state spaces. Notably, while our task-centric canonicalization module introduces a mathematically negligible coordinate transformation overhead ($0.045$s inference time vs. FlowMP's $0.038$s), the framework easily satisfies real-time control requirements and remains over $3\times$ faster than MPD.

\subsubsection{RQ3: Geometric Robustness and SE(2) Generalization}
On unseen $\mathrm{SE(2)}$ transfer tasks,baseline models suffer a catastrophic performance collapse under spatial shifts. MPD's VR decreases from $43.0\%$ to a mere $19.0\%$, while FlowMP crashes from $65.0\%$ to $21.0\%$. This confirms that standard architectures overfit to absolute coordinates, rendering them brittle under spatial shifts. In contrast, BridgeFlow exhibits remarkable geometric robustness, maintaining a VR of $80.0\%$ on these completely unseen tasks, which proves the efficacy of our canonicalization framework. By systematically factoring out global spatial symmetries prior to network inference, BridgeFlow achieves robust zero-shot generalization, obviating the need for expensive data augmentation in high-dimensional learning.

\section{Conclusion}
\label{Sec.5}
In this paper, we propose BridgeFlow, a novel $\mathrm{SE(2)}$-equivariant motion planning framework based on Flow Matching that resolves the bottlenecks of existing learning-based planners. By synergistically integrating a Brownian bridge informative prior, environment-aware Classifier-Free Guidance, and context-aware mini-batch optimal transport, our framework achieves exceptionally fast inference while establishing a highly informative geometric prior. Furthermore, a lightweight task-centric canonicalization module guarantees strict $\mathrm{SE(2)}$-equivariance, enabling robust generalization across continuous rigid body transformations.

Future work will focus on three directions: (1) Extending the framework to 3D motion planning by adapting the occupancy map to 3D distance fields and extending $\mathrm{SE(2)}$-equivariance to $\mathrm{SE(3)}$-equivariance. (2) Optimizing the noise variance \(\sigma_{\text{prior}}\) and guidance weight \(\omega\) using adaptive methods to further improve trajectory quality and inference speed. (3) Validating the framework on real robotic systems to demonstrate its practical applicability in real-time scenarios.

\addtolength{\textheight}{-12cm}   









\end{document}